\documentclass[12pt]{article}
\PassOptionsToPackage{pdfpagelabels=false}{hyperref} 
\usepackage[english]{babel}
\usepackage{amsmath, amsxtra, amsfonts, amssymb, amstext}
\usepackage{amsthm}
\usepackage{booktabs}
\usepackage{fullpage}
\usepackage{nicefrac}
\usepackage{xspace}
\usepackage[noadjust]{cite}
\usepackage{url}
\usepackage{graphics,color}
\usepackage[colorlinks]{hyperref}
\definecolor{linkblue}{rgb}{0.1,0.1,0.8}
\hypersetup{colorlinks=true,linkcolor=linkblue,filecolor=linkblue,urlcolor=linkblue,citecolor=linkblue}
\usepackage[algo2e,ruled,vlined,linesnumbered]{algorithm2e}
\usepackage{multirow}
\usepackage{pdfpages}

\usepackage{array}
\usepackage{booktabs}
\usepackage{bigdelim}

\newtheorem{theorem}{Theorem}
\newtheorem{lemma}[theorem]{Lemma}

\newtheorem{definition}[theorem]{Definition}

\newtheorem{remark}[theorem]{Remark}

\newcommand{\R}{\mathbb{R}}

\newcommand{\A}{\mathcal{A}}
\newcommand{\F}{\mathcal{F}}

\newcommand{\I}{\mathcal{I}}
\newcommand{\Ce}{\mathcal{C}}

\renewcommand{\epsilon}{\varepsilon}
\newcommand{\eps}{\varepsilon}

\DeclareMathOperator{\E}{\mathbb{E}}

\newcommand{\assign}{\leftarrow}

\newcommand{\oea}{$(1 + 1)$~EA\xspace}

\newcommand{\onemax}{\textsc{OneMax}\xspace}

\newcommand{\leadingones}{\textsc{LeadingOnes}\xspace}
\newcommand{\lo}{\textsc{Lo}\xspace}
\newcommand{\LO}{\textsc{Lo}\xspace}

\newcommand{\AM}{\ensuremath{\mathcal{A}_{\mathcal{M}}}}

\begin{document}

\title{The (1+1) Elitist Black-Box Complexity of LeadingOnes\footnote{An extended abstract of this paper will appear at GECCO 2016.}}

\author{
Carola Doerr$^{1,2}$, 
Johannes Lengler$^3$}
\date{
{\small $^1$CNRS, UMR 7606, LIP6, 75005 Paris, France\\
$^2$Sorbonne Universit\'es, UPMC Univ Paris 06,  UMR 7606, LIP6, 75005 Paris, France\\
$^3$Institute for Theoretical Computer Science, ETH Z{\"u}rich, Switzerland} 
\\[2ex]
\today
}
\maketitle

\begin{abstract}
One important goal of black-box complexity theory is the development of complexity models allowing to derive meaningful lower bounds for whole classes of randomized search heuristics. Complementing classical runtime analysis, black-box models help us understand how algorithmic choices such as the population size, the variation operators, or the selection rules influence the optimization time. One example for such a result is the $\Omega(n \log n)$ lower bound for unary unbiased algorithms on functions with a unique global optimum [Lehre/Witt, GECCO 2010], which tells us that higher arity operators or biased sampling strategies are needed when trying to beat this bound. In lack of analyzing techniques, almost no non-trivial bounds are known for other restricted models. Proving such bounds therefore remains to be one of the main challenges in black-box complexity theory.

With this paper we contribute to our technical toolbox for lower bound computations by proposing a new type of information-theoretic argument. We regard the permutation- and bit-invariant version of \textsc{LeadingOnes} and prove that its (1+1) elitist black-box complexity is $\Omega(n^2)$, a bound that is matched by (1+1)-type evolutionary algorithms. The (1+1) elitist complexity of \textsc{LeadingOnes} is thus considerably larger than its unrestricted one, which is known to be of order $n\log\log n$ [Afshani et al., 2013].
\end{abstract}


\sloppy{

\begin{table*}
\centering
\begin{scriptsize}
\begin{tabular}{l|ll|ll}
\textbf{Model}    
& \multicolumn{2}{c|}{\textbf{Lower Bound}} 
& \multicolumn{2}{c}{\textbf{Upper Bound}}\\ \hline 
unrestricted 
	& $\Omega(n \log\log n)$ & \cite{AfshaniADLMW12}
	& $O(n \log\log n)$ & \cite{AfshaniADLMW12} 
	\\ \hline
unbiased, arity $1$
	& $\Omega(n^2)$ & \cite{LehreW12}
	& $O(n^2)$ & \cite{Rudolph97} for (1+1)~EA
	\\ 
unbiased, arity $2$ 
	& $\Omega(n \log\log n)$ & 
	& $O(n \log n)$ & \cite{DoerrJKLWW11}
	\\ 
unbiased, arity $\geq 3$ 
	& $\Omega(n \log\log n)$ & 
	& $O(n \log (n)/\log\log n)$ & \cite{DoerrW11EA}
	\\ 	\hline	
ranking-based unbiased, arity $\geq 3$ 
	& $\Omega(n \log\log n)$ & 
	& $O(n \log (n)/\log\log n)$ & \cite{DoerrW11EA}
	\\ 	\hline		
elitist, arity $1$ 
	& $\Omega(n^2)$ & this paper
	& $O(n^2)$ & \cite{DrosteJW02}
	\\ 	\hline		
\end{tabular}
\end{scriptsize}
\caption[caption]{Known black-box complexities of \leadingones.  The lower bounds for higher arities follow from the lower bound in the unrestricted model.}

\label{tab:LOknown}
\end{table*}

\section{Introduction}
\label{sec:intro}

Since the seminal paper of Droste, Jansen, and Wegener~\cite{DrosteJW06} \emph{black-box complexity} is the most accepted complexity measure for black-box search heuristics such as evolutionary algorithms (EAs). Informally, the black-box complexity of a set $\F$ of functions is the minimum expected number of function evaluations that are needed to solve any problem instance $f \in \F$, where the minimum is taken over a class of algorithms $\A$. The original black-box model by Droste, Jansen, and Wegener regards as $\A$ the whole collection of possible black-box algorithms. It is therefore called the \emph{unrestricted black-box model}. 
However, unlike in classical complexity theory where a widely accepted complexity notion is used, several black-box complexity models co-exist in the theory of randomized search heuristics. 
Each model regards a different collection $\A$ of algorithms (e.g., the memory-restricted model regards only such algorithms that keep in the memory only a limited number of previously sampled search points while in contrast the algorithms in the unrestricted black-box are assumed to have full access to all previous function evaluations). 
When we compare the black-box complexity of a problem in the different models, we thus learn how certain algorithmic choices such as the population size, the variation operators in use, or the selection rules influence the performance of the search heuristics. 

At GECCO~2015 \cite{DoerrL15model} we have presented a new black-box model that combines several of the previously regarded restrictions such as the size of the memory and the selection rules. More precisely, we have defined a collection of $(\mu+\lambda)$ \emph{elitist black-box models}, where a $(\mu+\lambda)$ \emph{elitist algorithm} is one that keeps at any point in time the $\mu$ best-so-far sampled solutions. In the next iteration it is allowed to use only the relative (not absolute) function values of these $\mu$ points to create some $\lambda$ new search points. Truncation selection is then used to select from these $\mu+\lambda$ points the $\mu$ surviving ones that form the parent population of the next iteration. A $(\mu+\lambda)$ black-box algorithm is thus in particular a memory-restricted and ranking-based one in the sense of~\cite{DrosteJW06,DoerrW14memory} and~\cite{DoerrW14ranking}, respectively. In addition, it has to employ truncation selection as replacement rule. 

In~\cite{DoerrL15model} examples are presented for which the elitist black-box complexity is much larger than in any of the previously existing black-box models, while in~\cite{DoerrL15OM} it is shown that the complexity of \onemax is of linear order only even for the most restrictive (1+1) setting. Here in this work we regard another classic problem, a generalized version of the LeadingOnes function. We show that its (1+1) elitist black-box complexity matches the performance of typical evolutionary algorithms. 

\subsection{The LeadingOnes Problem}
\label{sec:lo}

LeadingOnes (\lo for short) is among the best-studied functions in the theory of evolutionary computation. It was originally designed in~\cite{Rudolph97} to disprove the conjecture of M\"{u}hlenbein~\cite{Muehlenbein92} that the expected runtime of the \oea on every unimodal function is $O(n \log n)$. While Rudolph showed experimentally that its expected runtime is $\Theta(n^2)$, this bound was proved a bit later in~\cite{DrosteJW02}. Quite exact bounds for the expected runtime of the \oea on \lo have been shown in~\cite{BottcherDN10, Sudholt13, Ladret05}. 

As argued in~\cite{DoerrW11EA}, most EAs behave symmetrically with respect to function representation, and it is therefore natural that the proven performance guarantees extend to the \emph{class of generalized \lo functions} which contains all pseudo-Boolean functions $f:\{0,1\}^n \rightarrow \R$ having a fitness landscape that is isomorphic to that of the original \lo function (which assigns to each bit string the number of initial ones, cf. Section~\ref{sec:formal} for precise definitions).  

The \oea qualifies as a (1+1) elitist algorithm in the sense of~\cite{DoerrL15model}. The above-mentioned runtime bounds therefore imply that the (1+1) elitist black-box complexity of \lo is of order at most $n^2$. We show that this is bound is tight.

\subsection{Discussion of Our Result}
\label{sec:our}

Our main result is summarized by the following statement. 
\begin{theorem}
\label{thm:LO}
The (1+1) elitist black-box complexity of \lo is $\Theta(n^2)$. 
\end{theorem}

As mentioned, this bound is matched by the average performance of the \oea. It is also matched by the expected runtime of Randomized Local Search (RLS), thus showing that these two simple strategies are asymptotically optimal for \lo among all (1+1)-type elitist algorithms. Our result also shows that an algorithm trying to beat the $\Omega(n^2)$ bound (and such algorithms exist, cf.~below) has to use larger population sizes or non-elitist selection strategies. 

In our proof we will not make use of the fact that elitist algorithms have to be ranking-based; that is, the (1+1) elitist black-box complexity of \lo remains $\Theta(n^2)$ even if the algorithms have access to the absolute fitness value of the current search point.

We summarize all known black-box complexities of \lo in Table~\ref{tab:LOknown}. Relevant for our contribution are in particular the tight 
$\Theta(n \log\log n)$ bound for the unrestricted black-box complexity of \lo obtained by Afshani, Agrawal, Doerr, Doerr, Larsen, and Mehlhorn~\cite{AfshaniADLMW12}, 
the $O(n \log n)$ bound for the binary unbiased black-box complexity by Doerr, Johannsen, K{\"o}tzing, Lehre, Wagner, and Winzen,~\cite{DoerrJKLWW11}, and the 
$\Omega(n^2)$ bound for the unary unbiased model by Lehre and Witt~\cite{LehreW12}. 
Note also that already simple binary search exhibits a complexity of only $O(n \log n)$ on the \lo problem (cf.~\cite{AfshaniADLMW12} for an implementation of the binary search strategy). This is why it is remarkable that (1+1) elitist algorithms have such a rather weak performance on \lo.

Our result is not the first lower bound of quadratic order for the \lo problem. In fact, Lehre and Witt proved in~\cite{LehreW12} that all unary unbiased search strategies, i.e., intuitively speaking, all black-box algorithms using only mutation as variation operators, need $\Omega(n^2)$ function evaluations on average to optimize this problem. 
Combining our result with theirs, we see that even if we replace the mutation operator in RLS (which flips exactly one bit, chosen uniformly at random) or the (1+1)~EA (which flips independently each bit with probability $1/n$) by a---possibly strongly---biased one, the resulting algorithm would still need time $\Omega(n^2)$ on average. This shows that not only unbiased sampling, but also the population structure of the algorithm and the selection strategies determine the comparatively slow convergence of these two well-known search heuristics.  

In addition to identifying such structural bottlenecks, our result is also---and this is in fact the main motivation for our studies---interesting from a purely mathematical point of view, as we need to develop some new tools for the lower bound proof. Specifically, we use some information-theoretic arguments, utilizing that the amount of information that the algorithm has at any given point is not sufficient to make substantial progress. Such information-theoretic arguments are notoriously hard to formulate rigorously, and are even harder to employ in a non-trivial situation like ours.
 
 What complicates our analysis is the fact that the \lo functions, in principle, allow for a rather huge storage. Indeed, when the fitness of an individual is $k$ for some $k<n$, then all but the $k+1$ bits determining the fitness of the search point can be used for storing information about previous samples, the number of iterations elapsed, or any other information gathered during the optimization process. We recall that such strategies of constantly storing information about previous samples are at the heart of many upper bounds in black-box complexity, cf., for example, the proofs of the $O(n /\log n)$ bound for the (1+1) memory-restricted~\cite{DoerrW14memory}, the ranking-based~\cite{DoerrW14ranking} or the (1+1) elitist Monte Carlo~\cite{DoerrL15OM} black-box complexity of \onemax. We therefore need to show that despite the fact that changing the $n-(k+1)$ \emph{irrelevant bits} has no influence on the fitness, the algorithm cannot make effective use of this storage space. 
 
The intuitive reason why the given storage is not large enough is that by the permutation-invariance of the \lo problem the black-box algorithms do not know where the irrelevant bits are located, and storing this information would require more bits than available. However, the discrepancy is rather small: in most parts of the process, if the number of bits of the storage space was larger by just a constant factor, then this would trivially allow for efficient use of the storage, and the lower bounds would break down. It is thus essential to find a good measure for the information that an algorithm can possibly encode in its queries.
We develop a precise notion that bounds the amount of information that the algorithm has about the function $\LO_{z,\sigma}$ at any given state. We can use this notion to estimate the gain that the algorithm can draw from any given amount of information. We consider one of the main contributions of our paper to make these intuitive concepts precise and utilizable in proofs. Although our definitions are adapted to the \lo problem, we are optimistic that the developed techniques are applicable also to other black-box settings and, of course, also to other problem classes.

A second difficulty that we face in our proof is that the the minimax principle of Yao~\cite{Yao77}---which is the foremost technique in black-box complexity theory to prove lower bounds---cannot be applied to the elitist model, as discussed in~\cite{DoerrL15model}. We therefore need to find an extension of the elitist model that is generous enough to allow for the application of the information-theoretic tool but that does, at the same time, not decrease the black-box complexity by too much.

\section{Formal Definitions}
\label{sec:formal}

\begin{algorithm2e}[t]
 \textbf{Initialization:} \\
 \Indp
 $X \assign \emptyset$\;
 \For{$i=1,\ldots,\mu$}{
 Depending only on the multiset $X$ and the ranking $\rho(X,f)$ of $X$ induced by $f$, choose a probability distribution $p^{(i)}$ over $\{0,1\}^n$ and sample $x^{(i)}$ according to $p^{(i)}$\;
 $X \assign X \cup \{ x^{(i)}\}$\;
 }
 \Indm
 \textbf{Optimization:}	
 \For{$t=1,2,3,\ldots$}{
 		Depending only on the multiset $X$ and the ranking $\rho(X,f)$ of $X$ induced by $f$
	\label{line:mut}	choose a probability distribution $p^{(t)}$ on $(\{0,1\}^n)_{i=1}^{\lambda}$ and 
		sample $(y^{(1)},\ldots,y^{(\lambda)}) $ according to $p^{(t)}$\;
%
		Set $X \assign X \cup \{y^{(1)},\ldots,y^{(\lambda)}\}$\;
  \lFor{$i=1,\ldots, \lambda$}{
  	\label{line:selection} Select $x \in \arg\min X$ and update $X \assign X \setminus \{x\}$}
	 }
 \caption{The $(\mu+\lambda)$ elitist black-box algorithm for maximizing an unknown function $f:\{0,1\}^n \rightarrow \R$}
\label{alg:elitist}
\end{algorithm2e}

\subsection{Black-box Complexity} 
We come to the formal definitions. A $(\mu+\lambda)$ elitist black-box algorithm is a (possibly randomized) algorithm that can be described by the framework of Algorithm~\ref{alg:elitist}. I.e., assume that a pseudo-Boolean function $f:\{0,1\}^n \to \R$ is given, the \emph{fitness function}. Then the algorithm maintains a multiset $X$ of $\mu$ search points, which is called the \emph{population}. It is initialized by sampling iteratively $\mu$ search points, which may only depend on the previous search points and the ranking induced on them by the fitness function $f$. Afterwards, in each round it samples $\lambda$ new search points (\emph{offspring}) based only on the search points in $X$ and their ranking with respect to $f$. It then forms the new population by choosing the $\mu$ search points from the old population and the offspring that have the largest fitness, where it may break ties arbitrarily. This process is repeated until a maximum of $f$ is sampled. 

The \emph{runtime} of a $(\mu+\lambda)$ elitist black-box algorithm $A$ on a pseudo-Booloean $f$ is the number of search points (\emph{queries}) that $A$ samples before it samples for the first time a maximum of $f$. The expected runtime of $A$ on a class $\mathcal F$ of pseudo-Boolean functions is the maximum expected runtime of $A$ on $f$, where $f$ runs through $\mathcal{F}$. The $(\mu+\lambda)$ elitist black-box complexity of $\mathcal F$ is the smallest expected runtime of a $(\mu+\lambda)$ elitist black-box algorithm on $\mathcal F$.
  
\begin{remark}
In~\cite{DoerrL15model} we distinguished between two different notions of runtime, which in turn lead to different notions of (elitist) black-box complexity: the \emph{Las Vegas} runtime of a black-box algorithm $A$ on a function $f$ is the expected number of steps until $A$ finds the optimum of $f$, while the \emph{$p$-Monte Carlo} runtime is the minimal number of steps $A$ needs in order to find the optimum of $f$ with probability at least $1-p$. It was shown that there may be an exponential gap between the resulting elitist black-box complexities. However, this is not the case for \lo. Formally, we show for \lo that the (1+1) elitist Las Vegas black box complexity is $\Omega(n^2)$, and for every constant $p>0$ the (1+1) elitist $p$-Monte Carlo black box complexity is $\Omega(n^2)$. Both statements are immediate consequences of Theorem~\ref{thm:runtime}.

In this paper, black-box complexity refers to the Las Vegas version unless specified otherwise.
\end{remark}

\subsection{LeadingOnes}
The original \lo function is defined via $\lo:\{0,1\}^n \rightarrow \R, x \mapsto \max \{ i \in \{0,1,\ldots,n\} \,\mid\, \forall j \leq i: x_j=1\}$. 
This is generalized to a permutation- and bit-invariant version in the following way. 
Let $S_n$ denote the set of permutations of $[n] := \{1, \ldots, n\}$ and 
let $[0..n]:=\{0,1, \ldots, n\}$. 
For $z \in \{0,1\}^n$ and $\sigma \in S_n$ we consider the function
\begin{align*} 
\LO_{z,\sigma}: \{0,1\}^n  \rightarrow & \ [0..n],\\
 x \mapsto & 
\max \{ i \in [0..n]\mid \forall j \leq i: z_{\sigma(j)} = x_{\sigma(j)}\}\,,
\end{align*}
so $\LO_{z,\sigma}(x)$ is the length of the longest common prefix of the search point $x$
and the \emph{target string} $z$ in the order $\sigma$. For $i<j$ we say that $\sigma(i)$ is \emph{more significant} than $\sigma(j)$. In particular, $\sigma(1), \ldots,\sigma(k)$ are the $k$ most significant bits, and $\sigma(n),\ldots, \sigma(n-k+1)$ are the $k$ least significant bits. 
The \lo problem is the problem of optimizing an unknown member of the class $\lo := \{\LO_{z,\sigma} \mid z \in \{0,1\}^n, \sigma \in S_n \}$. Clearly, the unique global optimum of $\LO_{z,\sigma}$ is $z$. For ease of notation we drop the subscript $\{z, \sigma\}$ in the following.

\subsection{Other Notation and Definitions} 
Throughout this work, we will write $\log x$ for the binary logarithm $\log_2 x$ and $\ln x$ for the natural logarithm of $x$. We say that an event $\mathcal{E}$ holds \emph{with high probability} if $\Pr[\mathcal E] \to 1$ for $n\to\infty$.

A \emph{fitness level} of a function $f$ is a maximal set of search points which have the same fitness. In particular, every function $\LO_{z,\sigma}$ has exactly $n$ different fitness levels.

\section{Proof of the Lower Bound}
\label{sec:LOlower}

To prove the desired $\Omega(n^2)$ bound we introduce a more generous model in which the algorithms have strictly more power than in the original elitist model (Sections~\ref{sec:LOmodelalternative}--\ref{sec:LOinformation}). We then show the $\Omega(n^2)$ lower bound for this more generous model in Section~\ref{sec:LOruntimelower}. 
Before we start with the technical details of the alternative model, we explain in Section~\ref{sec:LOproblems} our motivation for introducing it. 
We present a high-level overview of the proof in Section~\ref{sec:LOhighlevel}. 

\subsection{Challenges in Proving the Lower Bound}
\label{sec:LOproblems}

\textbf{Learning from search points with inferior fitness: } 
One may be tempted to believe that in the (1+1) elitist model we cannot learn information from search points that have stricty lower fitness than that of the current best one. This is indeed a tantalizing thought as such search points have to be discarded immediately and can therefore not influence the sampling distribution of the next query. However, one has to be very careful with such arguments. To illustrate why it fails in general, consider the following setting: assume that there is a search point $x$ from which we sample search point $y_1$ with some very small probability $\eps$ and search point $y_2$ otherwise; i.e., we sample $y_2$ with probability $1-\eps$. 
For the sake of the argument assume further that for all search points $z \neq x$ the probability to sample $y_1$ or $y_2$ is zero. If at some stage of the algorithm we happen to have $y_1$ in the memory we may then conclude that we must have been at $x$ in the previous step. Moreover, if $f(y_2) > f(x)$ then with probability $1-\eps$ we would have proceeded to $y_2$, and thus we would never have visited $y_1$ (as we cannot return to $x$ from a fitter search point). Therefore, by Bayes' theorem $f(y_2) \leq f(x)$ with probability at least $1-\eps$. Summarizing, although we have not visited $y_2$, we can deduce information about its fitness.

\textbf{Application of Yao's Principle: }
A tool that has proven to be extremely helpful in deriving lower bounds for black-box complexities is the so-called MiniMax Principle of Yao~\cite{Yao77}. 
All lower bounds that we are aware of directly or indirectly use (the easy direction of) this tool. In simple words, Yao's Principle allows us to restrict our attention to the performance of a best-possible deterministic algorithm on a random input. This is a lower bound for the expected performance of a best possible randomized algorithm for this problem. 
This principle is typically used with a very simple distribution $p$. Indeed, in most proofs $p$ can be chosen to be the uniform distribution. 

\begin{lemma}[Yao's Principle~\cite{Yao77,MotwaniR95}]
\label{lem:Yao}
Let $\Pi$ be a problem with a finite set $\I$ of input instances (of a fixed size) permitting a finite set $\A$ of deterministic algorithms. Let $p$ be a probability distribution over $\I$ and $q$ be a probability distribution over $\A$. Then, 
\begin{align}\label{eq:Yao}
	\min_{A \in \A} \E[T(I_p, A)] \leq \max_{I \in \I} \E[T(I,A_q)]\, , 
\end{align}
where $I_p$ denotes a random input chosen from $\I$ according to $p$, $A_q$ a random algorithm chosen from $\A$ according to $q$ and $T(I,A)$ denotes the runtime of algorithm $A$ on input $I$.
\end{lemma}

As was pointed out in~\cite{DoerrL15model}, the informal interpretation of Yao's principle as stated before Lemma~\ref{lem:Yao} does not apply to elitist algorithms. Since this is a crucial difficulty in our proofs, we explain this apparent contradiction in detail, even though a very similar example was given in~\cite{DoerrL15model}. 

Let $\I=\lo$, let $p$ be the uniform distribution over the inputs $\I$, and let $I_p$ as in Lemma~\ref{lem:Yao}.
Then any deterministic algorithm $A$ has a positive probability during the optimization of $I_p$ of getting stuck. Assume $x$ and $y$ are the first two search points that $A$ queries, and note that since $A$ is an elitist black-box algorithm, the choice of $y$ does not depend on the fitness of $x$ (although the example could easily be adapted to cover this case as well). If the fitness of $y$ is strictly smaller than that of~$x$, $y$~has to be discarded immediately and the algorithm is in exactly the same state as before, and will continue sampling and discarding $y$. Therefore, the expected runtime of $A$ on this fitness function is infinite. Since this holds for every deterministic algorithm (that is, every deterministic (1+1) elitist algorithm has an infinite expected runtime on a uniformly chosen \lo instance), the lower bound in (\ref{eq:Yao}) is infinite, too. 
However, of course there are randomized search strategies with finite expected runtime, e.g., RLS and the \oea (see Section~\ref{sec:lo}). 

To resolve this apparent discrepancy, note that Lemma~\ref{lem:Yao} makes a statement about algorithms that are a convex combination of deterministic ones. For typical classes of algorithms this describes exactly the class of randomized algorithm, since we can emulate every randomized algorithm by making all random coin flips in advance, and then choosing the deterministic algorithm whose decisions in each step agree with these coin flips. However, this only works if the algorithm is free to make a new decision in each step, e.g., if the algorithm may base its decision on the number of previous steps. However, (1+1) black-box algorithms (elitist or not) may not do so since they are \emph{memory-restricted}. Therefore, randomized (1+1) black-box algorithms cannot be in general written as convex combinations of deterministic (1+1) black-box algorithms.

These observations have a quite severe effect on our ability to prove lower bounds in elitist black-box models. Indeed, the only way we currently know is the approach taken below, where we consider a superset $\AM$ of algorithms such that every randomized algorithm in $\AM$ can be expressed as a convex combination of deterministic ones. A lower bound shown for this broader class trivially applies to all elitist black-box algorithms. In our case, we achieve the class $\AM$ by giving the algorithms access to enough memory to determine the current step (within a certain phase).

If applicable, Yao's principle allows us to restrict ourselves to deterministic algorithms, which are usually easier to analyze. In particular, we may use the following observation. 
\begin{remark}
\label{rem:uniform}
Assume that we run a deterministic algorithm $A$ on a problem instance $i$ that we have taken from the set of instances $\I$ uniformly at random, so $\Pr[i = c] = 1/|\I|$ for every $c\in \I$. 
Assume further that the first queries $q_1,\ldots,q_\ell$ of $A$ reduce the number of possible problem instances to some set $\Ce$. Then $\Pr[i=c \mid q_1,\ldots,q_\ell] = \Pr[i=c \mid i \in \Ce] = 1/|\Ce|$ for all $c \in \Ce$. In particular, each $c \in \Ce$ is equally likely to be the secret instance $i$. 
\end{remark}

\subsection{High-Level Ideas of the Proof}
\label{sec:LOhighlevel}
As explained above, we cannot use Yao's principle directly for the set of all elitist black-box algorithms. Instead, we use a larger class \AM~of algorithms, the definition of which is adapted to the special structure of the \lo problem. In this model, whenever an algorithm reaches fitness level $k$ for the first time, we reveal for a brief moment the position of the $k$ most significant bits. Note that by symmetry of the \lo function, an algorithm cannot discriminate between the less significant bits from its previous samples. Therefore, everything that the algorithm has learned previously is covered by this piece of information. Based on this information, we allow the algorithm to ``store'' whatever it wants in the $m=n-k$ least significant bits. After that, we occlude the information about the $k$ significant bits again, and the algorithm may only use its storage of size $m$. However, until the algorithm finds the next fitness level we allow it to keep track of all the search points that it visits on the current fitness level. In this way we create a class of algorithm for which Yao's principle allows us to restrict to deterministic algorithms. The details are spelled out in Section~\ref{sec:LOmodelalternative}.

The algorithm may be lucky and skip a fitness level, because the $(k+1)$-st significant bit in its search point is correct. However, this only happens with probability $1/2$. Otherwise, the algorithm can only carry over $m$ bits of information to the next level, cf.~Lemma~\ref{lem:infoatstart} for a precise statement. Crucially, if $m = \delta n$ for some small $\delta >0$, then this information is not enough to encode the positions of the $k$ most significant bits, which would require $\log \binom{n}{k} = \log \binom{n}{m} \approx m(1+\log(1/\delta))$ bits. One strategy of the algorithm might be to store as many of the insignificant bit positions as possible, so that it can test quickly whether one of these candidates is the next significant bit. However, whenever the algorithm wants to be certain (or ``rather certain'') that a specific bit $b$ is insignificant, then this decreases the available information about the remaining bits. This strategy might pay off if $b$ is the next significant bit, because then the algorithm reaches a new fitness level immediately. However, with a too high probability $b$ is not the next significant bit, and the algorithm is left with an even worse situation.

The key to the proof lies in the exact definition of the class \AM, and in a precise way to capture the rather vague notation of information used above. We start by defining \AM~in Section~\ref{sec:LOmodelalternative}. In Section~\ref{sec:LOrefinedcost} we show that by restricting to one-bit flips, we extend the runtime of an algorithm on the $k$-th fitness level by at most $m=n-k$. In Section~\ref{sec:LOinformation} we first give a precise definition of what we mean by information, and we define the quantity $\Phi(k,m,B)$ to be the minimum expected number of queries that an algorithm in $\AM$ using one-bit flips needs to advance a fitness level, if there are $k$ significant and $m$ insignificant bits and if the available information is $B$. In Lemma~\ref{lem:recursion} we then give a rather straight-forward recursion for the function $\Phi(k,m,B)$. Once the recursion is established, it is purely a matter of (somewhat tedious) algebra to derive the lower bound $\Phi(k,m,B) \geq \eps (k+m)(1-(\log B)/(2m))$ in Lemma~\ref{lem:lowerboundPhi}. Since the starting amount of information is $B \approx 2^m$, each algorithm using one-bits flips needs to spend expected time $\Phi(k,m,2^m) \approx \eps (k+m)/2$ on the corresponding fitness level (provided that it visits this level at all). Since using multi-bit flips can save us at most $m$ queries, a general algorithm in \AM~spends at least time $\approx \eps (k+m)/2 - m$ on this level, which is $\Omega(n)$ if $m \leq \delta n$ for a sufficiently small $\delta >0$. Thus there is a linear number of fitness levels, such that the algorithm spends an expected linear time on each of them, showing the $\Omega(n^2)$ runtime. The details of this concluding argument are found in Section~\ref{sec:LOruntimelower}.

\subsection{A More Generous Model}
\label{sec:LOmodelalternative}

We consider the set \AM~of all algorithms that can be implemented in the following model~$\mathcal{M}$:
 
Assume that the algorithm has queried some search points $x^{(1)}, \ldots, x^{(t)}$ with $\max_{i<t}\{\LO(x^{(i)})\}=k-1$ and $\LO(x^{(t)}) \geq k$. 
We say that the algorithm \emph{reaches a new fitness level} with the $t$-th query. In addition to letting the algorithm know that the fitness value of $x^{(t)}$ is strictly larger than the previous best search point, 
we reveal to the algorithm the first $k$ \emph{significant} positions $\sigma(1), \ldots, \sigma(k)$ and the corresponding bit values $z_{\sigma(1)}, \ldots, z_{\sigma(k)}$ of the target string. Note that from this information the algorithm can in particular infer that $\LO(x^{(t)}) \geq k$, but it does not learn the precise function value of $x^{(t)}$. Furthermore, it is not difficult to see that the information revealed to the algorithm contains everything about the unknown instance $(z,\sigma)$ that the algorithm could have collect so far (and typically it reveals much more information about the target instance than the information currently present to the algorithm).
We now allow the algorithm to ``revise'' its choice of the \emph{insignificant} $m:=n-k$ bits of $x^{(t)}$. 
That is, the algorithm may opt to change the entries in the positions $[n]\setminus \{\sigma(1), \ldots, \sigma(k)\}$, thus creating a new search point $\tilde{x}^{(t)}$.  

By construction, the fitness of $\tilde{x}^{(t)}$ is at least $k$. It is possibly strictly greater than $k$ in which case the algorithm may again revise the entries of the insignificant bits, now based on the first $k+1$ positions. This process continues until the fitness of the revised search point equals the number of significant bit positions that the algorithm has seen when creating it. 
For ease of notation, let us assume that $f(\tilde{x}^{(t)}) = k$. For clarity we emphasize that the algorithm never learns about the exact fitness value of its original choice $x^{(t)}$.

Starting from $y^{(0)} := \tilde{x}^{(t)}$, until it reaches a new fitness level $>k$ we allow the algorithm to remember all queries $y^{(1)}, \ldots, y^{(s)}$, and for each of them whether $f(y^{(i)})$ is smaller or whether it is equal to $f(y^{(0)})$ (if $f(y^{(i)})$ is larger than $f(y^{(0)})$, a new fitness level is reached). 
Moreover, we allow it to remember the value $k$. 
The algorithm may thus choose $y^{(s+1)}$ depending on $k$, $y^{(0)}, \ldots y^{(s)}$, and on the information which of the $y^{(i)}$ have smaller fitness than $y^{(0)}$. Crucially, note that in this phase the algorithm does no longer have access to the positions $\sigma(1),\ldots,\sigma(k)$ of the first $k$ significant bits, unless it has somehow encoded this information implicitly, for example in $y^{(0)}$.

We want to bound from below the expected number of queries that are needed to reach a new fitness level, that is, the expected number $T_m$ of queries before the algorithm queries a search point of fitness strictly larger than $k = n-m$. Note that this number may be zero if $f(\tilde{x}^{(t)}) >k$. 
We shall apply Yao's MiniMax Principle with the uniform distribution over the possible \lo instances $(z,\sigma)$. A discussion of this tool has been given in Section~\ref{sec:LOproblems}. 
We will show next that, unlike for the original elitist model, in $\AM$ every (reasonable) randomized algorithm is a convex combination of deterministic ones, the crucial difference to the original elitist model being that in $\AM$ the algorithms may remember the search points of the current fitness level. In particular, a reasonable deterministic algorithm therefore never ``gets stuck''. Formally, we get the following statement.

\begin{lemma}\label{lem:convexcomb}
Assume that $A\in \AM$ is a randomized algorithm that never samples a search point twice on the same fitness level. Then $A$ is a convex combination of deterministic algorithms in $\AM$.
\end{lemma}
Before we prove Lemma~\ref{lem:convexcomb}, we recall that a (deterministic or randomized) algorithm $A \in \AM$ is allowed to remember all previous queries on the same fitness level, so it will know whether it is about to sample the same search point twice on the same fitness level. Its runtime can only increase by sampling the same search point twice, so for proving lower bounds we can restrict to algorithms $A$ as in the lemma. The reason why we actually need this restriction is that there is only a finite number of possible executions (i.e., of sequences of queries) of $A$ if we forbid multiple sampling. Otherwise the number of possible executions would be uncountable, which causes some problems that we want to avoid.

\begin{proof}[Proof of Lemma~\ref{lem:convexcomb}]
Assume we have a randomized algorithm $A \in \AM$. All random
decisions of $A$ can be based on a sequence of random coin flips. (Or,
depending on the model of computation, on a sequence of real random
variables in the interval $[0,1]$. Here we will assume coin flips.)

For simplicity, we will first argue that for every $N>0$, the algorithm
$A$ in the phase before the $N$-th coin flip can be obtained as a convex
combination of (or, more precisely, a probability distribution
over) deterministic algorithms. In other words, if we only consider the
execution of $A$ up to the $N$-th coin flip,  then the randomized
algorithm is just a randomly chosen deterministic one. In the following
paragraph we will thus only regard the execution of $A$ until it
requests the $N$-th coin flip.

We can emulate $A$ as follows. At the very beginning (before
actually running $A$), we flip $nN$ coins, which we denote by $F_{i,j}$,
where $1\leq i \leq n$ runs through all fitness levels, and $1\leq j
\leq N$. We now run $A$ as follows. When it is on the $i$-th fitness
level and it asks for the $j$-th random bit on this fitness level, then
we feed it the bit $F_{i,j}$. Note that although there are $Nn$ random
bits available in total, $A$ will use at most $N$ of them (in the part
of the execution that we consider).

Now that the $F_{i,j}$ are fixed, the algorithm $A$ is just a
deterministic algorithm. The subtle point here is that we can also
realize it as a deterministic algorithm $AÕ$ in the class
$\mathcal{A}_M$. This is because whenever the algorithm $A'$ is in a
specific state (i.e, it has a specific current search point with
specified fitness, and it has some content in the memory that is allowed
for the class $\mathcal{A}_M$), then the available information suffices
to determine the number of queries in this fitness level. In particular,
regardless of what $A'$ has done in previous steps, the current state is
different from all previous states. (This is the crucial difference to
the elitist model, in which a deterministic algorithm may come to the
same state twice.) So for all possible continuations for $A$ in the same
state, there is a deterministic algorithm that follows either
continuation. Therefore, once the $F_{i,j}$ are fixed we can select a
deterministic algorithm $AÕ$ from $\AM$ that behaves like $A$
for the first $N$ coin flips.

For unlimited $N$ we use the same argument as before, only that we flip an infinite
sequence of coins for each fitness level. Still, every outcome $(F_{i,j})_{1\leq i \leq n, j \geq 1}$
corresponds to exactly one deterministic algorithm, which also does not query the same search point twice on the same fitness level. Hence, every such
deterministic algorithm is chosen with some probability, where
the probabilities add up to $1$. The argument now runs as before.
\end{proof}

We show that there is a constant $\eps >0$ such that for all $1 < m \leq \eps n$ and all deterministic algorithms in $\AM$ the expected number of queries that the algorithm spends on fitness level $k=n-m$ is at least $\eps n$. 
This yields the desired $\Omega(n^2)$ lower bound. 
By the following lemma, this bound also holds for all elitist $(1+1)$ algorithms. 

\begin{lemma}
Every elitist $(1+1)$ algorithm is also in $\AM$.
\end{lemma}
\begin{proof}
This follows rather trivially from the definition of $\AM$. An elitist $(1+1)$ algorithm can be simulated by an algorithm in $\AM$ by choosing $\tilde x^{(t)} := x^{(t)}$, and by ignoring all information except for the current search point. Note that for a (1+1) elitist algorithm it suffices that the oracle tells it which of the two search points it compares is the better one, or whether they are of equal fitness. The (1+1) elitist algorithm will thus not know more about the search points $y^{(i)}$ than whether the fitness is worse, equal, or better than the fitness of $y^{(0)}$. Thus the (1+1) elitist algorithm has always at most the information that an algorithm in $\AM$ has.
\end{proof}

\subsection{It Suffices to Study Single Bit-Flips}
\label{sec:LOrefinedcost}

One technical challenge in bounding the complexity of \lo with respect to $\AM$ is the question of how to deal with multiple bit flips. The following lemma tells us that we do not give away much if we restrict ourselves to algorithms that only use one-bit flips. This observation simplifies the upcoming computations significantly. 
%

\begin{lemma}
\label{lem:coins}
For every deterministic algorithm $A \in \AM$ there exists a deterministic algorithm $A'\in \AM$ such that the following holds. 
If for some instance algorithm $A$ uses $s$ queries to leave fitness level $k$, then $A'$ uses at most $s+n-k$ queries on fitness level $k$. 
Moreover, all the search points $y'^{(1)}, \ldots, y'^{(r)}$ that $A'$ uses on this fitness level have Hamming distance one from $y^{(0)}$.
\end{lemma}

\begin{proof}
Let $A \in \AM$ be deterministic, and assume it queries $y^{(0)},y^{(1)}, \ldots, y^{(s)}$ on fitness level $k$. 
The algorithm $A'$ also starts with $y^{(0)}$. 
For $i \in [s]$ let $\ell_i \in [0..n]$ be such that $y^{(i)}$ differs from $y^{(0)}$ in $\ell_i$ bits. 
In the $i$-th \emph{step}, which may consist of several queries, $A'$ goes through these $\ell_i$ bits, flipping them one by one (always starting with $y^{(0)}$) and querying the resulting strings until it finds a string that has fitness smaller than $k$, or until it has exhausted the $\ell_i$ bits. 
If $A'$ creates a string that it queried in one of the previous $i-1$ steps, it does not query this string again. Note that this is possible in the model $\AM$, but would not be possible in the (1+1) elitist model.

We need to show two things: 
\textbf{(i)} after each step, $A'$ has at least as much information as $A$ has, so that it knows which query $y^{(i+1)}$  algorithm $A$ will choose next, and 
\textbf{(ii)} $A'$ uses at most $s+m$ queries, where $m=n-k$. 
In fact, we show that each of the $m$ insignificant bits causes at most one additional query for $A'$. Assume first that $A'$ exhausts the $\ell_i$ bits in step $i$, i.e., that all the $\ell_i$ corresponding strings have fitness at least $k$. In this case $A$ learns that the fitness of $y^{(i)}$ is at least $k$ (and possibly that it is strictly larger than $k$), and $A'$ learns the same. Moreover, $A$ has done one query, while $A'$ has done one query for each of the at most $\ell_i$ insignificant bits that have not been evaluated before. 

Next assume that in the $i$-th step $A'$ finds a string that has fitness smaller than $k$. Say it is the $\ell'$-th string that $A'$ queries in the $i$-th step. Then $A$ only learns that among the $\ell_i$ flipped bits there is at least one significant bit. $A'$ learns this, too, but it also learns specifically the position of such a bit. 
Algorithm $A$ uses one query, while $A'$ uses $\ell'\leq \ell_i$ queries, spending again one query for each insignificant bit that it discovers. 

In both cases, algorithm $A'$ learns at least all the information that $A$ learns, and the total number of queries done by $A'$ is at most $s+m$ since it spends at most one query for each insignificant bit that it tests.
\end{proof}

%

It is now easy to argue that with Lemma~\ref{lem:coins} at hand we need to consider only algorithms that do single bit-flips.
We show below that, for $m$ being within a suitable range of linear size, every such algorithm in expectation needs at least $m + \eps n$ queries to leave fitness level $k:=n-m$, provided that it started with a search point $y^{(0)}$ of fitness exactly $k$. 
Lemma~\ref{lem:coins} implies that every other algorithm (possibly doing multiple bit-flips) needs in expectation at least $\eps n$ queries to leave level $k$, provided that it visits this level. 

\subsection{Evolution of the Available Information}
\label{sec:LOinformation}

In this section we study how the amount of information that a deterministic algorithm has about the problem instance $(z,\sigma)$ changes over time, in particular while the algorithm stays on one fitness level.  

Let us consider first how much information the algorithm has when entering a new fitness level $k=n-m$. 
Recall that we consider a problem instance $(z,\sigma)$ that is taken from all \lo functions uniformly at random. 
Remember also that in our model, i.e., model $\AM$ described in Section~\ref{sec:LOmodelalternative}, we reveal to the algorithm the value $k$ of the fitness level, 
the position of the $k$ significant bits $\sigma(1), \ldots, \sigma(k)$, 
and the values $z_{\sigma(1)}, \ldots, z_{\sigma(k)}$ of the corresponding bits. 
In our model $\AM$ we allow the algorithm to change the entries in the $m$ insignificant positions. Intuitively, we thus implicitly grant it $m$ bits for storing information about the problem instance. In the following, we make this intuition precise.

Let a \emph{$k$-configuration} be a pair $(P,u)$ of a subset $P$ of $[n]$ of size $k$ and a $\{0,1\}$-valued string~$u$ of length~$k$. 
We interpret $(P,u)$ as the set $\{\sigma(1), \ldots, \sigma(k)\}$ of the first $k$ significant bit positions, together with the values $z_{\sigma(1)}, \ldots, z_{\sigma(k)}$ of these bits in the optimum. Thus a $k$-configuration (together with the value of $k$) describes exactly the information the algorithm has before choosing the revised search point $\tilde x^{(t)}$, when it leaves the $(k-1)$-st fitness level. 
The (deterministic) algorithm maps each such possible $k$-configuration $(P,u)$ to a bit string $\tilde x^{(t)}$ of length $n$. 
However, since there are $2^k\binom{n}{k}$ $k$-configurations and only 
$2^{n}$ bit strings, there are on average at least $2^{k-n}\binom{n}{k} = 2^{-m}\binom{k+m}{m}$ different $k$-configurations that are matched to the same string $\tilde x^{(t)}$. In the following, we will track the number $C$ of $k$-configurations that are still compatible with the history of the algorithm on level $k$, i.e., that are compatible with the fact that $f(y^{(0)})=k$, with the fact that the algorithm has chosen $y^{(0)}$, and with the oracle's answers to $y^{(1)},\ldots,y^{(i)}$, for some $i\geq 0$. Note that Remark~\ref{rem:uniform} applies, i.e., all these $k$-configurations are equally likely to be the problem instance.

Assume that the algorithm starts the $k$-th fitness level with some string $y^{(0)}$ ($= \tilde{x}^{(t)})$. Then the compatible $k$-configurations $(P,u)$ are determined by $y^{(0)}$ and the set $P$. This follows from the fact that by construction the entries in the $k$ significant positions in the string $y^{(0)}$ coincide with the optimal ones $z_{\sigma(1)}, \ldots, z_{\sigma(k)}$ (these bits were not changed when creating $\tilde{x}^{(t)} = y^{(0)}$). Since the algorithm has access to $y^{(0)}$ anyway, a compatible configurations can be described by the $k$ significant positions. There are $\binom{n}{k}=\binom{k+m}{m}$ sets of size $k$ in $[n]$, so it is convenient to normalize by this factor. 
We thus define 
\begin{align}
\label{eq:Bdef}
B := B(k,m,C) = \frac{\binom{k+m}{m}}{C}
\end{align}
as the factor by which the number of possible target configurations have been reduced already. We call $B$ the \emph{available information} after querying $y^{(i)}$. Note that always $B \geq 1$. We remark that information is often measured in bits, which would correspond to $\log B$. However, for our purposes it is more convenient to work with $B$ rather than $\log B$.

\begin{lemma}
\label{lem:infoatstart}
With probability at least $1/2$, the available information after querying $y^{(0)}$ is at most $2^{m+1}$. 
\end{lemma}
\begin{proof}
Recall that the algorithm matches on average $2^{-m}\binom{k+m}{m}$ different $k$-configurations to each string $\tilde x^{(t)}$. Let $\Ce$ be the set of all strings $\tilde x^{(t)}$ which correspond to at most $2^{-m-1}\binom{k+m}{m}$ $k$-configurations. 
These strings together cover at most $2^{n-m-1}\binom{k+m}{m}=2^{k-1}\binom{k+m}{m}$ $k$-configurations, i.e., at most half of all $k$-configurations. 
Since the $k$-configuration of the problem instance is drawn uniformly at random, with probability at least $1/2$ we draw a configuration that belongs to a string in $\{0,1\}^n\setminus \Ce$. Thus, with probability at least $1/2$ we hit a string which is mapped to a $\tilde x^{(t)}$ that is compatible with more than $2^{-m-1}\binom{k+m}{m}$ $k$-configurations. This proves the claim.
\end{proof}

Let us now consider how the information evolves with the queries on the $k$-th fitness level.
Let $\mathcal{A}_{\text{one}}$ be the set of all deterministic algorithms that, starting from an $n$-bit string $y^{(0)}$ with $\LO_{z,\sigma}(y^{(0)})=k$, use only one-bit flips of $y^{(0)}$ until they have found a string $y^{(s)}$ with $\LO_{z,\sigma}(y^{(s)})>k$. 
\begin{definition}
\label{def:phi}
For any $k \geq 0$, $m\geq 1$, and $B\geq 1$ we define 
$\Phi := \Phi\!(k,m,B)$ 
to be the minimal expected number of fitness evaluations that an algorithm $A \in \mathcal{A}_{\text{one}}$ needs in order to find the next fitness level on a string with $k$ significant and $m$ insignificant bits if the instance $(z,\sigma)$ is chosen uniformly at random among a set $\Ce$ of $k$-configurations. Here the minimum is taken over all algorithms $A \in \mathcal{A}_{\text{one}}$ and all sets $\Ce$ with $|\Ce| \geq C(k,m,B):=\binom{k+m}{m}/B$. 
For convenience we set $\Phi(k,0,B) := 0$ for all $k$ and $B$. 
\end{definition}

Note that $\Phi(k,m,B)$ is a decreasing function in $B$. Before we study $\Phi(k,m,B)$ in detail, let us first compare $\Phi(k,m,B)$ with the expected time needed by \emph{any} algorithm in $\AM$ (i.e., not necessarily based on single bit-flips) to reach a new fitness level. 
When an algorithm $A \in \AM$ exceeds fitness $k-1$ and chooses $\tilde x^{(t)}$, with probability $1/2$ it holds that $\LO_{z,\sigma}(\tilde x^{(t)})=k$ and with probability $1/2$ the function value of $\tilde x^{(t)}$ is strictly larger than $k$. This is by the uniform choice of the problem instance. Moreover, by Lemma~\ref{lem:infoatstart}, with probability at least $1/2$ the available information is at most $B \leq 2^{m+1}$, where $m=n-k$, and this event is independent of whether $\LO_{z,\sigma}(\tilde x^{(t)})=k$. In particular, with probability at least $1/4$, we have both $\LO_{z,\sigma}(\tilde x^{(t)})=k$ and $B \leq 2^{m+1}$. 
In this case, by Lemma~\ref{lem:coins}, the expected time that $A$ spends on the $k$-th fitness level is at least $\Phi(k,m,2^{m+1})-m$. Thus, our aim will be to show that $\Phi(k,m,2^{m+1})-m = \Omega(n)$ for a linear number of values of $m$. 
We start our investigations with a recursive formula for $\Phi$.

\begin{lemma}
\label{lem:recursion}
Let $k\geq 0$, $m\geq 1$, and $B\geq 1$. Then
\[
\Phi\!(k,m,B) \geq \frac{m+1}{2}.
\] 
Furthermore, for $p_{\min}:= \max\{0,1-Bk/(k+m)\}$ and $p_{\max}:= \min\{1,Bm/(k+m)\}$ it holds that
\begin{align}
\label{eq:recursionii}
\Phi\!(k,m,B)  & \geq  1+ \min_{p\in [p_{\min},p_{\max}]} \left\{ p\frac{m-1}{m}\Phi\!\left(k,m-1,\frac{B}{p}\cdot \frac{m}{k+m}\right) \right.\\ 
& \hspace{10em}\left.+  (1-p) \Phi\!\left(k-1,m,\frac{B}{1-p}\cdot \frac{k}{k+m}\right)\right\},\nonumber
\end{align}
where we use the convention that for $p=0$ ($p=1$) the first (second) summand of the minimum evaluates to zero.
\end{lemma}

\begin{proof}
For the first formula, simply observe that even if the algorithm knows the configuration exactly, it still needs to test the $m$ insignificant bits one by one (by definition of $\mathcal{A}_{\text{one}}$) until it finds the next significant one, i.e., until the fitness improves. 
Recalling that the position of the next significant bit is uniformly among the insignificant ones, the expected number of steps that it takes the algorithm to find it is $(m+1)/2$.

To verify~\eqref{eq:recursionii}, let $A \in \mathcal{A}_{\text{one}}$, and assume that the set $\Ce$ of configuration compatible with the algorithm's choice of $y^{(0)}$ satisfies $|\Ce| \geq \binom{k+m}{m}/B$. We need to show that for each such $A$ and $\Ce$ the expected number of remaining fitness evaluations to find the next fitness level is at least the right hand side of \eqref{eq:recursionii}. Assume further that in its next query, $A$ flips the bit $b_i$, yielding a search point $y^{(1)}$, and let $p\in[0,1]$ be such that exactly $p|\Ce|$ configurations are compatible with the event $f(y^{(1)}) \geq f(y^{(0)})$. 
Since all configurations are equally likely, $p$ is also the probability that $f(y^{(1)}) \geq f(y^{(0)})$. 
Moreover, the number of such configurations is at most $\binom{k+m-1}{m-1}$, 
since $b_i$ has to be one of the insignificant bit positions. 
This shows that $p$ is bounded by the inequality
$p|\Ce| \leq \binom{k+m-1}{m-1}$. Using $|\Ce| \geq \binom{k+m}{m}/B$ this implies $p \leq Bm/(k+m)$. Similarly, $(1-p)|\Ce| \leq \binom{k+m-1}{m}$, which implies $p \geq 1-Bk/(k+m)$.

Assume that the event $f(y^{(1)}) \geq f(y^{(0)})$ happens. Then with probability $1/m$ the algorithm leaves the $k$-th fitness level (since all $m$ insignificant bits have the same probability of being the next significant bit). 
Otherwise, that is, with probability $(m-1)/m$, the algorithm learns that $b_i$ is not the next significant bit, and it will not query $b_i$ again on this level. Therefore, we may just exclude it from our considerations, and replace $m$ by $m-1$. Since the number of remaining compatible configurations is $p|\Ce|$, we need to find $B_{\text{new}}$ such that $\binom{k+m-1}{m-1}/B_{\text{new}} \leq  p|\Ce|$. Since $p|\Ce| \geq  p\binom{k+m}{m}/B$, we may choose $B_{\text{new}}$ to satisfy $\binom{k+m-1}{m-1}/B_{\text{new}} =  p\binom{k+m}{m}/B$, or equivalently $B_{\text{new}} = (B/p) \cdot (m/(k+m))$. So if $f(y^{(1)}) \geq f(y^{(0)})$ then the algorithm needs at least an expected additional time of $\frac{m-1}{m}\Phi\!\left(k,m-1,\frac{B}{p}\cdot \frac{m}{k+m}\right)$.

Now we consider the case $f(y^{(1)}) < f(y^{(0)})$, which happens with probability $1-p$. Then the algorithm learns that $b_i$ is significant, and it will not query $b_i$ again on this level. Thus we can exclude it from our considerations and replace $k$ by $k-1$. Similar as before, the number of compatible configurations drops to $(1-p)|\Ce|$, so we need to find $B_{\text{new}}$ such that $\binom{k+m-1}{m}/B_{\text{new}} \leq (1-p)|\Ce|$. We may choose  $B_{\text{new}}$ according to the equation $\binom{k+m-1}{m}/B_{\text{new}} = (1-p)\binom{k+m}{m}/B$ and obtain $B_{\text{new}} = (B/(1-p)) \cdot (k/(k+m))$. 
So if $f(y^{(1)}) < f(y^{(0)})$ then the algorithm needs at least an expected additional time of 
$\Phi\!\left(k-1,m,B/(1-p)\cdot k/(k+m)\right)$.  
This proves~\eqref{eq:recursionii}. 
\end{proof}

We use Lemma~\ref{lem:recursion} to show the following lower bound for $\Phi(k,m,B)$. Once this bound is proven, we have everything together to prove the claimed $\Omega(n^2)$ bound for the (1+1) elitist black-box complexity of \lo. 

\begin{lemma}
\label{lem:lowerboundPhi}
There exists a constant $\eps >0$ such that for all $k \geq 0, m\geq 1$ and $B\geq 1$,
\begin{align}
\label{eq:lowerboundPhi}
\Phi\!(k,m,B) \geq \eps (k+m)\left(1-\frac{\log B}{2m}\right).
\end{align}
\end{lemma}

\begin{proof}
We use induction on $k+m$. First we show the statement for the case $m=1$ and arbitrary $k$. 
If $m=1$ and $B \geq 4$, the lower bound is at most zero, and thus trivial. 
If $m=1$ and $B < 4$, by~\eqref{eq:Bdef}, the number of compatible configurations is at least $(k+1)/B > (k+1)/4$, and in each of these configurations there is exactly one insignificant bit. 
Since these bits are all different from each other, there are at least $(k+1)/4$ positions at which the insignificant bit might be, and by Remark~\ref{rem:uniform} all these positions are equally likely. 
Since the algorithm is by definition only allowed to make one-bit flips, it needs in expectation at least $(k+1)/8$ steps. Thus, the statement is satisfied for all $\eps \leq 1/8$.  

Now we come to the inductive step, where by the paragraph above we may assume that $m \geq 2$. Furthermore, we may also assume that $\log B < 2m$ since the statement is trivial otherwise. By Lemma~\ref{lem:recursion} we have 
\begin{align}
\label{eq:recursion}
\Phi\!(k,m,B) 
& \geq  1 + \!\min_{p\in [p_{\min},p_{\max}]} \left\{p\frac{m-1}{m}\Phi\!\left(k,m-1,\frac{B}{p}\cdot \frac{m}{k+m}\right) \right. \\ 
& \hspace{10em}+ \left.  (1-p)\cdot \Phi\!\left(k-1,m,\frac{B}{1-p}\cdot \frac{k}{k+m}\right)\right\}. \nonumber\end{align}
By the induction hypothesis we may thus conclude from inequality (\ref{eq:recursion}) that
\begin{align*}
\Phi\!(k,m,B)  & \geq  1+ \!\min_{p\in [p_{\min},p_{\max}]} \left\{
 p\frac{m-1}{m} \eps (k+m-1) \big(1-\frac{\log(\frac{B}{p}\cdot \frac{m}{k+m})}{2(m-1)}\big) \right.\\
& \left.\hspace{10em} + (1-p) \eps (k+m-1) \big(1-\frac{\log(\frac{B}{1-p}\cdot \frac{k}{k+m})}{2m}\big)\right\}\\
&  = 
1+ \eps (k+m) \left(1-\frac{\log B}{2m}\right) 
+
\min_{p\in [p_{\min},p_{\max}]} \left\{ 
- R(p)
\right\},\\
\end{align*}
where $R(p)$ equals 
\begin{align*}
R(p)& = &&  \frac{1}{m} p \eps  (k+m-1)    \left(1-\frac{\log\big(\frac{B}{p}\cdot \frac{m}{k+m}\big)}{2(m-1)}\right) 
&\textbf{$=:X_1$} \hfill\\
& 
+ && p \eps   \left(1-\frac{\log\big(\frac{B}{p}\cdot \frac{m}{k+m}\big)}{2(m-1)}\right) 
& \textbf{$=:A_1$}\hfill\\
& 
+ && p \eps  (k+m) \frac{\log\big(\frac{1}{p}\cdot \frac{m}{k+m}\big)}{2m} &\textbf{$=:Y_1$}\hfill\\
& 
+ && p \eps  (k+m) \frac{\log(\frac{1}{p}\cdot \frac{m}{k+m})}{2m(m-1)} &\textbf{$=:Z$}\hfill\\
& 
+ && p \eps (k+m)  \frac{\log B}{2m(m-1)} &\textbf{$=:X_2$}\hfill\\
& 
+ && (1-p) \eps  \left(1-\frac{\log(\frac{B}{1-p}\cdot \frac{k}{k+m})}{2m}\right) 
&\textbf{$=:A_2$}\hfill\\
& 
+ && (1-p) \eps  (k+m) \frac{\log(\frac{1}{1-p}\cdot \frac{k}{k+m})}{2m} &\textbf{$=:Y_2$}
\end{align*}
The above expression for $R(p)$ can be verified by elementary calculations, using only that $\log(B x) = \log B + \log x$ for $x = 1/p \cdot m/(k+m)$ and for $x = 1/(1-p) \cdot k/(k+m)$. It thus suffices to show that 
\begin{align}
\label{eq:LO-toshow}
R(p) \leq 1 \text{ for all } p\in [p_{\min},p_{\max}].
\end{align}
To this end we first observe that 
$\frac{B}{p}\cdot \frac{m}{k+m} \geq 1$ and 
$\frac{B}{1-p}\cdot \frac{k}{k+m} \geq 1$ by definition of $p_{\min}$ and $p_{\max}$. Since these expressions occur $X_1$, $A_1$, and $A_2$, we get 
\begin{align}
\label{eq:LO-A}
A_1 + A_2  & \leq \eps.
\end{align}
For the same reason, $X_1 \leq p\eps(k+m-1)/m$. Recalling $\log B < 2m$ and using $m-1 \geq m/2$, we thus get
\begin{equation}\label{eq:X}
X_1+X_2 \leq 3 p\eps(k+m)/m =: X
\end{equation}
In the following we will show that $Y_1+Y_2+ X+Z \leq 1/2$ if $\eps$ is sufficiently small. Together with \eqref{eq:LO-A} and \eqref{eq:X}, this will complete the inductive step. 

Let $p_0:= m/(k+m)$, and write $p = \gamma p_0$ for some (not necessarily constant) $\gamma \geq 0$. Note that $X = 3\gamma \eps$. We first consider the case $\gamma \leq 2$. In this case, $X \leq 6 \eps$. Moreover, $Z$ is either negative (for $\gamma >1$) or upper bounded by $-\eps\gamma\log \gamma \leq \eps/(e\ln 2)$, since the function $-\gamma \log \gamma$ has a unique maximum $1/(e\ln 2)$ in the interval $0<\gamma\leq 1$. Either way, we can choose $\eps$ so small that $X+Z \leq 1/2$. In the following we will prove that $Y_1+Y_2 \leq 0$, which will settle the case $\gamma \leq 2$. 

To prove that $Y_1+Y_2$ is non-positive, we may as well consider the sign of the function 
\[
f(p) := \frac{2m}{(k+m)\eps}(Y_1+Y_2) = p \log\left(\frac{p_0}{p}\right) + (1-p) \log\left(\frac{1-p_0}{1-p}\right).
\]
The function has derivatives
\[
f'(p) = \log\left(\frac{p_0}{p}\right) - \log\left(\frac{1-p_0}{1-p}\right) 
\]
and
\[
f''(p) = -\frac{1}{p(1-p)\ln(2)} < 0.
\]
Since the second derivative is negative, the function $f$ is concave. 
We further observe that $f(p_0) = f'(p_0) =0$. 
Thus $f$ has a unique maximum at $p=p_0$. Therefore, $f(p)\leq f(p_0) = 0$ for all $0\leq p\leq 1$. Hence, we have proven $Y_1+Y_2 \leq 0$, and this concludes the case $\gamma \leq 2$.

For the case $\gamma >2$, we first show that $Y_1/2 +X+Z \leq 1/2$. 
Clearly we have $Z\leq 0$. For sufficiently small $\eps >0$ we obtain 
\[
\frac{Y_1}{2} + X = -\frac{\eps \gamma\log\gamma}{4} + 3 \eps \gamma = \eps\gamma \left(3-\frac{\log\gamma}{4}\right) \leq \frac{1024\eps}{e\ln 2} \leq 1/2,
\]
where the second to last inequality can be easily checked by taking the derivative of the function and observing that is has a global maximum for $\gamma = 4096/e$.

Next we show that $Y_1/2 +Y_2 \leq 0$ for all $\gamma > 2$. Similar as before, we may consider the sign of the function
\begin{align*}
\tilde f(p) & := \frac{2m}{(k+m)\eps}(Y_1/2+Y_2)  = \frac{p}{2} \log\left(\frac{p_0}{p}\right) + (1-p) \log\left(\frac{1-p_0}{1-p}\right)
\end{align*}
under the additional constraint $2p_0 < p \leq 1$. Similar as above, the second derivative of $\tilde f$ for $0<p<1$ is
\[
\tilde f''(p) = -\frac{(p+1)}{2\ln(2)p(1-p)} < 0.
\]
Thus $\tilde f$ is concave for $0 \leq p\leq 1$. Recall that $\log(1+ x) < x$ for all $x> 0$. Therefore, 
\begin{align*}
\tilde f(2p_0) 
& = 
-p_0\log(2)+(1-2p_0)\log\left(1+\frac{p_0}{1-2p_0}\right) < -p_0+p_0 =0.
\end{align*}
On the other hand, $\tilde f(p_0) = 0 > \tilde f(2p_0)$. Since $\tilde f$ is concave, this implies $\tilde f(p) <0$ for all $2p_0<p<1$. This shows that $Y_1/2 +Y_2 \leq 0$ for all $\gamma > 2$. 

Summarizing, we have shown that $Y_1/2 + Y_2 \leq 0$ and $Y_1/2 +X+Z \leq 1/2$ for all $\gamma >2$. Together with \eqref{eq:LO-A} and \eqref{eq:X}, this concludes the inductive step and the proof. 
\end{proof}

\subsection{Putting Everything Together}
\label{sec:LOruntimelower}

As outlined in the high-level overview, Lemma~\ref{lem:coins},~\ref{lem:infoatstart},~and \ref{lem:lowerboundPhi} together imply runtime $\Omega(n^2)$ on the \LO problem for any algorithm in $\AM$. More precisely, we obtain the following theorem.

\begin{theorem}\label{thm:runtime}
With $\eps$ as in Lemma~\ref{lem:lowerboundPhi} and $k$ satisfying $1 < n-k \leq \eps n/8$, every algorithm $A\in \AM$ spends in expectation at least $\eps n/32$ function evaluations on level $k$, and this lower bound holds independently of the time spent on previous levels. In particular, every algorithm in $\AM$ (and thus, every elitist (1+1) black-box algorithm) needs at least time $\Omega(n^2)$ in expectation and with high probability.
\end{theorem}
\begin{proof}
We have already argued above after Definition~\ref{def:phi} that each deterministic algorithm $A\in \AM$ spends in expectation at least time $(\Phi\!(k,m,2^{m+1})-m)/4$ on fitness level $k = n-m$, since with probability $1/2$ the algorithm does not skip the level, with probability $1/2$ the available information is at most $2^{m+1}$ (Lemma~\ref{lem:infoatstart}), and conditioned on both these events, by Lemma~\ref{lem:coins}, $A$ spends at least expected time $\Phi\!(k,m,2^{m+1})-m$ on the $k$-th fitness level. By Lemma~\ref{lem:convexcomb} we may apply Yao's principle to deduce that the same bound also holds for every randomized algorithm $A\in \AM$.

The lower bound follows immediately from Lemma~\ref{lem:lowerboundPhi} which for $k$ and $m$ satisfying $1<m=n-k\leq \eps n/8$ yields
\[
\Phi\!(k,m,2^{m+1})-m \geq \eps n\left(1-\frac{m+1}{2m}\right) -m \geq \frac{\eps n}{4}- \frac{\eps n}{8} = \frac{\eps n}{8}.
\]
Note that this lower bound holds independently of the time spent on other fitness levels because in the model~$\mathcal{M}$ every algorithm which enters the $k$-th fitness level is in exactly the same state, i.e., it has access to exactly the same information, independent of its history. Since the lower bound holds for all algorithms in \AM, it still holds if we condition on the history of the algorithm, or specifically on the time spent on previous fitness levels. 

The lower bound $\Omega(n^2)$ on the expected runtime follows immediately. For the statement with high probability, it follows by a mostly formal argument that every algorithm spends at least linear time on each level with probability $\Omega(1)$. Note that this implies the statement, since then by the Chernoff bound with high probability every algorithm spends at least linear time on a linear number of levels. Since the formal argument is somewhat tricky, we elaborate on it. 

We want to show that there are constants $c,p>0$ such that for sufficiently large $n$, every algorithm in $\AM$ spends at least time $cn$ with probability at least $p$ on level $k$. Assume otherwise for the sake of contradiction, so assume that for every constant $c,p>0$ there is an algorithm $A\in\AM$ and there are arbitrarily large $n$ such that $A$ spends at least time $cn$ with probability at most $p$. Then as a formal consequence the same holds for $c,p=o(1)$, so there are functions $c=c(n)=o(1)$ and $p=p(n)=o(1)$ such that there is an algorithm $A\in\AM$ and arbitrarily large $n$ for which $A$ spends at least time $cn$ with probability at most $p$. Fix such an algorithm $A$, and consider the algorithm $A'\in \AM$ that behaves like $A$ for the first $cn$ queries, and afterwards just does random single bit flips. Note that the latter strategy takes expected time at most $n$ to leave fitness level $k$. Therefore, $A'$ needs in expectation at most $cn + pn = o(n)$ queries to leave level $k$ (for arbitrarily large values of $n$), contradicting the fact that every algorithms in $A$ needs expected time $\Omega(n)$. This concludes the formal argument, and thus the proof.
\end{proof}

\begin{remark}\label{rem:freememory}
Theorem~\ref{thm:runtime} can be strengthened in the following way. Assume that an elitist $(1+1)$ algorithm has an additional memory of size $\eps' n$ which it may use without restrictions. Then if $\eps'$ is sufficiently small, the algorithm still has expected runtime $\Omega(n^2)$. 

On the other hand, if the algorithm has in addition $n+O(\log n)$ bits of memory that it may use without restriction, then it is possible to achieve a runtime of $O(n\log n)$. Hence, if the algorithm has access to $cn$ bits of unrestricted memory, then it depends on the constant $c$ whether runtime $o(n^2)$ is possible or not. 
\end{remark}

\begin{proof}
The proof of the first statement is identical to the proof of Theorem~\ref{thm:runtime}, except that we increase the available information $B$ by a factor $2^{\eps' n}$. E.g., for $\eps' = \eps /16$ the algorithm still spends an expected linear time on all levels with $\eps n/12 \leq m \leq \eps n/8$.

The second statement holds because it is possible to store all $k$ significant bits in $n$ bits of the additional memory, and the remaining $O(\log n)$ bits may serve as a counter. Then, whenever the fitness increases, the algorithm performs $O(\log n)$ steps to determine (one of) the position(s) which was responsible for the improvement. The details are as follows.

We split the additional memory into a block $B_1$ of length $n$, and the remaining part $B_2$. Let $x$ denote the current search point, and let $\ell$ be the current number of one-bits in $B_1$. We maintain the invariant that every one-bit in $B_1$ is at the position of one of the $f(x)$ leading bits. Note that this implies that every such position must be correct in $x$. At start we choose $B_1$ to be the all-zero string. \\
We iteratively proceed as follows. If $\ell = f(x)$, then the bits in $B_1$ correspond exactly to the first $f(x)$ leading bits. In this case we flip all positions in $x$ that correspond to zero-bits in $B_1$, and this operation improves the fitness of $x$ by at least one. In the same operation, we (re-)set $B_2$ to zero. On the other hand, assume $\ell \neq f(x)$. Then our invariant implies $\ell < f(x)$. Let us call $P$ the set of all bit positions that are correct in $x$ but are not one-bits in $B_1$. Our aim is to identify one of the $f(x)-\ell$ positions in $P$. Let $P_0$ be the set of positions of zero-bits in $B_1$. In our strategy $P_0$ will serve as a set of candidate positions, which we will shrink iteratively. We create a search point $x'$ by flipping the first half of the positions $P_0$ in $x$, and query $f(x')$. If $f(x') > f(x)$ then we (have to) accept the new search point, and we reset all of $B_2$ to zero. If $f(x')<f(x)$ then we know that at least one of the positions in $P$ lies in the first half of $P_0$, so we replace $P_0$ by its first half, and encode this result by a single bit in $B_2$. Finally, if $f(x') = f(x)$ then at least one (in fact, all) of the positions in $P$ lie in the second half of $P_0$, so we replace $P_0$ by its second half, and also encode this result by a single bit in $B_2$. Repeating this step, we iteratively decrease the size of $P_0$ by a factor of $2$, or find a better search point. We continue this procedure for at most $\lceil \log n \rceil$ rounds, after which we have either found a better search point, or reduced $P_0$ to a single position, which then must be a position in $P$. In the latter case, we flip the corresponding bit in $B_1$ to one, reset $B_2$ to zero, and continue.\\
In this way, in $\log n$ steps, we can increase either the fitness of $x$, or the number of ones in $B_1$. Since we never flip any bit in $B_1$ to zero, after $O(n \log n)$ steps, either $f(x) = n$ or $B_1$ is the all-one string. By the invariant the latter also implies $f(x)=n$, so the algorithm finds the optimum after $O(n \log n)$ steps.
\end{proof}

\section{Conclusions and Outlook}
We have shown that the (1+1) elitist black-box complexity of \lo is $\Theta(n^2)$. This is in contrast to the situation for the \onemax function, where elitist selection does not substantially harm the running time~\cite{DoerrL15OM}. Given the much smaller complexity of \lo in many other models, this sheds some light on the cost of elitism. In fact, our proof suggests that the reason for the large complexity is rather the memory restriction than the selection strategy. We thus conjecture that the lower bound in Theorem~\ref{thm:LO} holds already for (1+1) memory-restricted algorithms, but we do not see at the moment a feasible proof for this claim. In particular the generalization of Lemma~\ref{lem:coins}, i.e., the statement that it suffices to consider one-bit flips, seems tricky.

Our methods are adapted to the \lo problem, and it is probably non-trivial to transfer them to other problems. Nevertheless, we are optimistic that the similar approaches can work both for other black-box complexities and other function classes.

Finally, in light of~\cite{DoerrDE13} we are interested to use the insights from our investigations for the design of new search heuristics. 

\subsection*{Acknowledgments}
This research benefited from the support 
of the ``FMJH Program Gaspard Monge in optimization and operation research'', 
and from the support to this program from \'Electricit\'e de France.

}
\newcommand{\etalchar}[1]{$^{#1}$}

\end{document}